\documentclass[11pt]{article}
\usepackage{sustyle}
\usepackage{tikz}
\usepackage{iftex}
\usetikzlibrary{decorations.text,shapes,arrows.meta}
\usepackage[most]{tcolorbox}

\newtcolorbox{researchquestion}[1]{
  colback=white,
  colframe=black!45,
  boxrule=0.5pt,
  arc=1pt,
  left=7pt,
  right=7pt,
  top=6pt,
  bottom=6pt,
  before skip=8pt,
  after skip=8pt,
  title=\textbf{#1},
  fonttitle=\normalsize
}

\usepackage[hidelinks]{hyperref}
\usepackage{hyperref}
\hypersetup{
    colorlinks=true,       
    linkcolor=blue,        
    citecolor=blue,        
    urlcolor=blue          
}

\usepackage{setspace}
\usepackage{pgfplots}
\usepackage{bchart}

\usepackage{verbatim}

\pgfplotsset{compat=1.18}
\usetikzlibrary{
    arrows.meta,
    calc,
    positioning
}

\newcommand{\citep}{\cite}
\newcommand{\citet}{\cite}

\title{When Is Muon Optimal? Radial Fragility and Minimax Geometry}

\begin{document}

\author{Ao Xu\\[0.5em]
\small Jilin University, Changchun 130012, China\\
\small Zhongguancun Academy, Beijing 100094, China\\[0.25em]
\small \texttt{xuao24@mails.jlu.edu.cn, s-xa24@bza.edu.cn}}
\date{August 5, 2026}

\maketitle

\begin{abstract}
Muon replaces a matrix gradient by an approximate polar factor and has
shown strong empirical performance in neural-network training. The
Isotropic Curvature Model (ICM) provides a curvature-based explanation
for this update by asking when gradient orthogonalization minimizes a
one-step local-loss surrogate. We investigate the stability of this
explanation when the fixed input radius of the ICM is replaced by a
general radial law.
We introduce the Random-Radius Isotropic Curvature Model (RR-ICM) and
the Polar-Factor Admissibility Set, which encodes exact scaled-polar
optimality as a geometric condition on the positive polar factor of the
gradient. We derive an exact characterization of this set for arbitrary
radial laws and show that all nonscalar admissible responses arise from
radial atoms aligned with curvature kinks. Consequently, every atomless
radial law exhibits zero-capacity rigidity: exact scaled-polar
optimality is possible only when the gradient is already a scaled
column isometry.
To connect the curvature-based framework with the established spectral
geometry of Muon, we introduce the Minimax RR-ICM by replacing
directional averaging with worst-case optimization over probability
laws on the sphere. We prove that every nonzero global minimizer lies
on the polar ray of the gradient, while its scale is determined by a
one-dimensional convex problem. The corresponding admissibility
geometry is a positive-semidefinite trace slab, in sharp contrast to
the spherical-zonoid geometry of the RR-ICM. Together, these results
expose the radial fragility of average-curvature explanations and
provide a new minimax variational characterization of Muon's polar
direction.
\end{abstract}

\section{Introduction}
\label{sec:introduction}
Muon is a matrix-aware optimizer designed primarily for hidden-layer weight
matrices. Given a momentum-smoothed matrix gradient
\[
M_t=U_t\Sigma_tV_t^\top,
\]
Muon applies a small number of Newton--Schulz iterations to approximate
its polar factor $U_tV_t^\top$, while non-matrix parameters such as
embeddings, normalization parameters, biases, and output heads are
typically handled by AdamW~\cite{jordan2024muon}.
Consequently, Muon preserves the left and right singular subspaces of
the momentum update while largely flattening its nonzero singular
values. In this sense, it replaces a potentially highly anisotropic
matrix update by an approximately semi-orthogonal one. 
Muon first attracted attention through strong empirical performance in
small-scale language-model and image-classification training, and was
subsequently shown to scale to substantially larger language models.
Liu et al.~\cite{liu2025muon} identified weight decay and
parameter-shape-dependent update scaling as key ingredients for stable
large-scale training, and developed a memory- and
communication-efficient distributed implementation. Their scaling-law
experiments reported that Muon required approximately $52\%$ of the
training FLOPs needed by a tuned AdamW baseline at matched performance.
They further used Muon to train Moonlight, a $3\mathrm{B}/16\mathrm{B}$
mixture-of-experts model, on $5.7$ trillion tokens. These results
suggest that matrix-gradient orthogonalization is not merely a
small-model heuristic, but can provide meaningful optimization gains
at the scale of modern language-model training.
\begin{figure}[t]
    \centering
    \includegraphics[width=\textwidth]{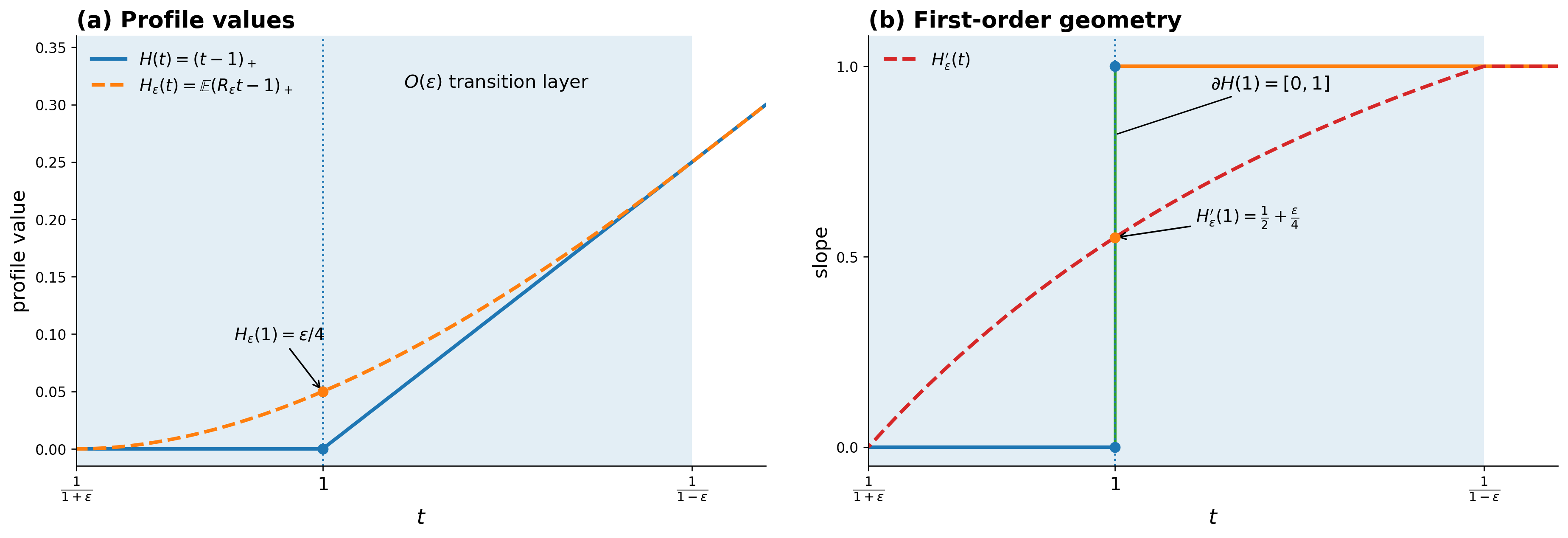}
    \caption{
        Radial averaging is a small perturbation at the level of function
        values but a singular perturbation at a kink. For $H(t)=(t-1)_+$ and
        $R_\varepsilon\sim\operatorname{Unif}
        [1-\varepsilon,1+\varepsilon]$, the averaged profile
        $H_\varepsilon(t)=\mathbb{E}(R_\varepsilon t-1)_+$ differs from $H$
        only within an $O(\varepsilon)$ transition layer. Nevertheless, the
        interval $\partial H(1)=[0,1]$ is replaced by the unique derivative
        $H_\varepsilon'(1)=1/2+\varepsilon/4$.
    }
    \label{fig:kink-radial-smoothing}
\end{figure}

Theoretical interest in this transformation extends beyond Muon itself.
Earlier work on stochastic and preconditioned spectral descent already
explored matrix updates adapted to spectral geometry~\cite{carlson2015stochasticRBM,carlson2015preconditioned}.
More recent studies have interpreted Muon through several closely
related viewpoints, including spectral-norm steepest descent and
architecture-dependent duality maps~\cite{bernstein2024old,bernstein2024modular},
norm-constrained linear minimization oracles~\cite{pethick2025training},
non-Euclidean trust-region optimization~\cite{kovalev2025understanding},
and implicit spectral-norm constraints~\cite{chen2025muon}.
These perspectives establish why the polar factor is a natural or
optimal update under a specified or implicitly induced spectral
geometry.
Related work has also connected Muon to matrix-aware preconditioning
and structured curvature surrogates
~\cite{lau2025polargrad,du2026newtonmuon}.
Su~\cite{su2025isotropic} took a complementary approach.
Rather than prescribing a norm geometry, constraint set, or trust
region, he derived the Isotropic Curvature Model (ICM) as a one-step
variational surrogate for the local change in a neural-network training
loss, under an isotropy assumption on its curvature across perturbation
directions. Within this surrogate framework, he asked under what
curvature conditions gradient orthogonalization—and, in particular,
Muon's polar update—emerges as an exact optimal matrix update.
Given a gradient matrix $G\in\mathbb{R}^{m\times n}$, the ICM considers
\begin{equation}
    J_{\mathrm{sph}}(Q;G)
    =
    -\langle G,Q\rangle
    +
    \mathbb{E}_{\zeta}H(\|Q\zeta\|_2),
    \qquad
    \zeta\sim\operatorname{Unif}(\mathbb{S}^{n-1}), \notag
\end{equation}
where $H$ is a scalar curvature profile. If $G=PS$ is the thin polar
decomposition and $Q=cP$, then $\|Q\zeta\|_2=c$ for every
$\zeta\in\mathbb S^{n-1}$. Hence, when $H$ has a kink at $c$, all spherical
samples lie at the same nondifferentiable point. Direction-dependent
subgradient selections may then generate a nonscalar matrix response capable
of supporting a scaled polar update.
The fixed-sphere assumption is often motivated by high-dimensional radial
concentration~\cite{van2014probability}. For example, a standard Gaussian vector satisfies
\begin{equation*}
    \frac{Z}{\sqrt n}
    =
    \rho_n \zeta,
    \qquad
    Z\sim\mathcal{N}(0,I_n),
    \qquad
    \rho_n=\frac{\|Z\|_2}{\sqrt n},
\end{equation*}
where $\zeta$ is uniform on $\mathbb S^{n-1}$ and independent of $\rho_n$,
and $\rho_n\to1$ in probability. For smooth observables, this suggests that
the normalized-Gaussian and fixed-sphere models should be asymptotically
close. 
Exact polar optimality in the ICM, however, relies on a set-valued
subdifferential at a kink, and radial concentration need not preserve this
first-order variational geometry. A deterministic radius can place every
sample exactly on the kink, whereas an atomless random radius hits any
prescribed kink with probability zero, no matter how concentrated its law is.
Thus, a thin random shell and a fixed sphere may be arbitrarily close in
radius while differing discontinuously in the subdifferential structure that
governs exact optimality.
Figure~\ref{fig:kink-radial-smoothing} illustrates this distinction for a
scalar hinge: radial averaging changes function values only within a shrinking
transition layer, yet replaces a set-valued subgradient by a unique smooth
response.

Motivated by the preceding observation, we introduce a natural
extension of the ICM in which the input radius is allowed to vary while
the directional component remains uniform on the sphere. We call this
the \emph{Random-Radius Isotropic Curvature Model} (RR-ICM):
\begin{equation}
    J_R(Q;G)
    =
    -\langle G,Q\rangle
    +
    \mathbb{E}_{R,\zeta}H(R\|Q\zeta\|_2),
    \qquad
    \zeta\sim\operatorname{Unif}(\mathbb{S}^{n-1}),
    \qquad
    R\perp\!\!\!\perp \zeta. \notag
\end{equation}
The model randomizes only the input radius while preserving rotational
symmetry. 

This extension leads to our first structural question.

\begin{researchquestion}{Question 1: Exact polar optimality under radial variation}
When the fixed-sphere ICM is extended to a general radial law, under
what conditions can a scaled polar factor $cP$ remain an exact global
minimizer? Can these conditions be characterized completely in terms of
the radial law, the curvature profile, and the spectrum of the
gradient's positive polar factor?
\end{researchquestion}

To answer this question, we introduce the
\emph{Polar-Factor Admissibility Set}, which encodes exact scaled-polar
optimality as a geometric condition on the positive polar factor of the
gradient. For an arbitrary radial law, we prove the exact
representation
\[
A(J_R,cP)=b_cI_n+w_c\mathcal K_n,
\]
where $\mathcal K_n$ is a spherical zonoid and $w_c$ records the total
atom--kink resonance at scale $c$. This formula gives a complete
geometric characterization of the gradient spectra compatible with
exact polar optimality. In particular, all nonscalar spectral freedom
is generated by radial atoms aligned with curvature kinks; when the
radial law is atomless, the admissibility set collapses to a single
scalar matrix.
This rigidity reveals a precise limitation of the average-curvature
explanation. Although radial concentration may make a thin shell close
to a sphere at the level of function values, exact polar optimality for
a general nonscalar gradient spectrum depends on a fragile
atom--kink resonance. Thus, the average RR-ICM alone does not provide a
radially robust variational explanation of Muon's polar update.

This conclusion creates an explanatory tension. Muon's polar direction
has broad spectral-geometric interpretations and performs robustly
across large-scale training settings~\cite{liu2025muon}, whereas exact polar optimality in
the average RR-ICM can be destroyed by an arbitrarily thin atomless radial perturbation. This leads to our second structural question.

\begin{researchquestion}
{\textbf{Question 2: A curvature-based bridge to spectral geometry}}
Does there exist a natural counterpart of the RR-ICM that retains its
local-curvature and random-radius structure, while selecting the polar
factor as the optimal matrix direction for general gradient spectra?
Can such a formulation establish a variational connection between the
ICM viewpoint and the established spectral interpretations of Muon?
\end{researchquestion}

To address this question, we replace the prescribed uniform
directional law in the RR-ICM by a worst-case choice over all
probability laws on the sphere. Writing
$\mathcal P(\mathbb S^{n-1})$ for the set of Borel probability measures
on $\mathbb S^{n-1}$, we introduce the \emph{Minimax RR-ICM}
\[
J_R^{\mathrm M}(Q;G)
=
-\langle G,Q\rangle
+
\sup_{\nu\in\mathcal P(\mathbb S^{n-1})}
\mathbb E_{\zeta\sim\nu}
\overline H_R(\|Q\zeta\|_2).
\]
We prove that, whenever the Minimax RR-ICM admits a nonzero global
minimizer, every such minimizer lies on the polar ray generated by the
gradient: its operator-norm-normalized matrix direction is exactly the
polar factor of $G$. The Minimax RR-ICM therefore provides
a new curvature-based minimax variational characterization of Muon's
polar update. Moreover, because it is obtained from the RR-ICM by
changing only how directional variation is aggregated, this
characterization establishes a direct variational bridge between the
ICM viewpoint and the established spectral-geometric interpretations
of Muon.

\paragraph{Contributions.}
Our main contributions are summarized as follows.
\begin{enumerate}
    \item \textbf{A polar-factor admissibility framework.}
    We introduce the \emph{Polar-Factor Admissibility Set}, which
    reformulates exact optimality of a scaled polar update as a
    convex-set membership condition on the positive polar factor of
    the gradient. This provides a common variational language for
    comparing different curvature models.

    \item \textbf{Exact polar-factor admissibility in the RR-ICM.}
    For the Random-Radius Isotropic Curvature Model, we characterize
    completely the gradient spectra compatible with exact scaled-polar
    optimality. For an arbitrary radial law, we prove
    \[
    A(J_R,cP)=b_cI_n+w_c\mathcal K_n,
    \]
    where $\mathcal K_n$ is a spherical zonoid and $w_c$ records the total
    atom--kink resonance at scale $c$. This representation shows that
    radial atoms aligned with curvature kinks are the sole source of
    nonscalar admissible responses, and yields a spectral-capacity
    bound and an exact rigidity result for atomless radial laws.

    \item \textbf{A new minimax variational characterization of Muon.}
    We introduce the Minimax RR-ICM by replacing directional averaging
    with worst-case optimization over spherical probability laws. We
    prove that every nonzero global minimizer lies on the polar ray of
    the gradient, while its scale is determined by an exact
    one-dimensional convex optimization problem.

    \item \textbf{A geometric bridge between curvature and spectral
    interpretations.}
    We characterize the Minimax Polar-Factor Admissibility Set as a
    positive-semidefinite trace slab and place it in direct comparison
    with the spherical-zonoid geometry of the average RR-ICM. The two
    models thereby reveal distinct mechanisms of polar optimality:
    atom--kink resonance under directional averaging and
    operator-norm subdifferential geometry under directional minimax
    robustification.
\end{enumerate}

\paragraph{Organization.}
Section~\ref{sec:rr-icm} formulates the RR-ICM and isolates the effect of radial
randomness through an effective curvature profile. Section~\ref{sec:main-results} gives the
exact characterization of the Polar-Factor Admissibility Set and
derives the resulting spectral-capacity and atomless-rigidity
consequences. Section~\ref{sec:optimal} introduces the Minimax RR-ICM,
establishes its polar optimizer geometry, and compares its admissibility
set with that of the average RR-ICM. The proofs are collected in
Section~\ref{sec:proofs}, and Section~\ref{sec:discussion} concludes with a discussion of the scope,
interpretation, and limitations of the framework.

\section{The Random-Radius Isotropic Curvature Model}
\label{sec:rr-icm}
Because the Gaussian radius concentrates around one in high dimensions, the
fixed-sphere and normalized-Gaussian models become close at the level of
ordinary function values. This section explains why that closeness need not
preserve first-order variational structure at a kink, and then formulates the
random-radius extension of the ICM.

\subsection{Why a Thin Shell Is Not a Sphere at a Kink}
\label{subsec:thin-shell-kink}

Consider the hinge profile
\begin{equation*}
    H(t)=(t-1)_+,
    \qquad t\geq 0.
\end{equation*}
and replace the deterministic unit radius by the atomless perturbation
\begin{equation*}
    R_\varepsilon
    \sim
    \operatorname{Unif}[1-\varepsilon,1+\varepsilon],
    \qquad
    0<\varepsilon<1.
\end{equation*}
The corresponding radially averaged profile is
\begin{equation*}
    H_\varepsilon(t)
    =
    \mathbb{E}(R_\varepsilon t-1)_+
    =
    \frac{1}{2\varepsilon}
    \int_{1-\varepsilon}^{1+\varepsilon}
    (rt-1)_+\,dr.
\end{equation*}
A direct calculation gives
\begin{equation}
    H_\varepsilon(t)
    =
    \begin{cases}
        0,
        &
        0\leq t\leq \dfrac{1}{1+\varepsilon},
        \\[2ex]
        \dfrac{\bigl((1+\varepsilon)t-1\bigr)^2}
        {4\varepsilon t},
        &
        \dfrac{1}{1+\varepsilon}
        <
        t
        <
        \dfrac{1}{1-\varepsilon},
        \\[3ex]
        t-1,
        &
        t\geq \dfrac{1}{1-\varepsilon}.
    \end{cases}
    \label{eq:averaged-hinge-closed-form}
\end{equation}
Thus, radial averaging modifies the hinge only on the transition interval
\begin{equation*}
    \left[
        \frac{1}{1+\varepsilon},
        \frac{1}{1-\varepsilon}
    \right].
\end{equation*}
The width of this interval is
\begin{equation*}
    \frac{1}{1-\varepsilon}
    -
    \frac{1}{1+\varepsilon}
    =
    \frac{2\varepsilon}{1-\varepsilon^2}
    =
    O(\varepsilon).
\end{equation*}
The perturbation is therefore small at the level of function values. Indeed,
because the hinge is \(1\)-Lipschitz, for every fixed \(T<\infty\),
\begin{equation*}
    \sup_{0\leq t\leq T}
    |H_\varepsilon(t)-H(t)|
    \leq
    T\mathbb{E}|R_\varepsilon-1|
    =
    \frac{\varepsilon T}{2}.
\end{equation*}

Its first-order geometry is nevertheless qualitatively different.
Differentiating~\eqref{eq:averaged-hinge-closed-form} yields
\begin{equation*}
    H_\varepsilon'(t)
    =
    \begin{cases}
        0,
        &
        0<t\leq \dfrac{1}{1+\varepsilon},
        \\[2ex]
        \dfrac{(1+\varepsilon)^2-t^{-2}}
        {4\varepsilon},
        &
        \dfrac{1}{1+\varepsilon}
        <
        t
        <
        \dfrac{1}{1-\varepsilon},
        \\[3ex]
        1,
        &
        t\geq \dfrac{1}{1-\varepsilon}.
    \end{cases}
\end{equation*}
At the original kink,
\begin{equation*}
    \partial H(1)=[0,1],
    \qquad
    H_\varepsilon'(1)
    =
    \frac{1}{2}+\frac{\varepsilon}{4}.
\end{equation*}
Hence, for every \(\varepsilon>0\), the entire subgradient interval of the
deterministic hinge is replaced by a single slope, even though
\(H_\varepsilon\to H\) locally uniformly as \(\varepsilon\to0\).

Figure~\ref{fig:kink-radial-smoothing} visualizes this distinction. A thin
random shell and a fixed sphere may be close in radius, but they charge a kink
in fundamentally different ways: the deterministic sphere places all of its
mass at the same transition, whereas an atomless radial law assigns zero mass
to every prescribed radius. Radial concentration therefore controls profile
values without necessarily preserving the first-order variational geometry
relevant to optimization.
The distinction is amplified in the matrix ICM. At a scaled polar point, the
fixed-sphere model can place every direction on the same kink and use the
resulting set-valued slopes to support a nonscalar matrix subgradient. Once
the radius becomes atomless, this direction-dependent slope-selection
mechanism may disappear. This motivates retaining the uniform spherical
direction while allowing the radius itself to be random.

\subsection{RR-ICM Formulation and Assumptions}
\label{subsec:rr-icm-model}

Let $m\geq n$, and let $G,Q\in\mathbb R^{m\times n}$ denote the
gradient signal and the matrix update, respectively. We write $\langle A,B\rangle:=\operatorname{tr}(A^\top B)$
for the Frobenius inner product. Let
$\zeta\sim\operatorname{Unif}(\mathbb S^{n-1})$, and let $R\geq0$ be a
random radius independent of $\zeta$.
Given a scalar curvature profile
$H:[0,\infty)\to\mathbb R$, the RR-ICM is defined by
\begin{equation}
    J_R(Q;G)
    :=
    -\langle G,Q\rangle
    +
    \mathbb E_{R,\zeta}
    H\!\left(R\|Q\zeta\|_2\right). \notag
\end{equation}
The fixed-sphere ICM is recovered as the special case $R\equiv1$.

It is convenient to absorb the radial law into an effective scalar
curvature profile. Define
\begin{equation}
    \overline H_R(t)
    :=
    \mathbb E_R H(Rt),
    \qquad
    t\geq0. \notag
\end{equation}
Since $R\geq0$ and is independent of $\zeta$,
\begin{align*}
    \mathbb E_{R,\zeta}
    H\!\left(R\|Q\zeta\|_2\right)
    &=
    \mathbb E_\zeta
    \left[
        \mathbb E_R
        H\!\left(R\|Q\zeta\|_2\right)
    \right]
    \\
    &=
    \mathbb E_\zeta
    \left[
        \overline H_R(\|Q\zeta\|_2)
    \right].
\end{align*}
Consequently, the RR-ICM admits the equivalent representation
\begin{equation}
    J_R(Q;G)
    =
    -\langle G,Q\rangle
    +
    \mathbb E_\zeta
    \left[
        \overline H_R(\|Q\zeta\|_2)
    \right].
    \notag
\end{equation}
Thus, radial randomness enters the objective entirely through the
scalar transformation
\[
    H\longmapsto\overline H_R,
\]
while the uniform spherical directional law remains unchanged.

Whenever a polar decomposition is invoked, we assume that $G$ has full
column rank and write
\begin{equation}
    G=PS,
    \qquad
    P^\top P=I_n,
    \qquad
    S=(G^\top G)^{1/2}\succ0.
    \notag
\end{equation}
No rank restriction is imposed on the optimization variable $Q$. 

\paragraph{Standing assumptions.}
Unless stated otherwise, $H:[0,\infty)\to\mathbb R$ is finite, convex,
and nondecreasing; $R\geq0$ is independent of
$\zeta\sim\operatorname{Unif}(\mathbb S^{n-1})$; and the effective
profile $\overline H_R$ is finite on $[0,\infty)$. Whenever scaled-polar
optimality is considered, $G\in\mathbb R^{m\times n}$ is assumed to
have full column rank, with $m\geq n$. These conditions retain the
baseline convexity, monotonicity, and rank assumptions used in Su's
analysis of orthogonalization in the fixed-sphere ICM
~\cite{su2025isotropic}. 

\section{Characterization of the Polar-Factor Admissibility Set}
\label{sec:main-results}
This section gives a complete geometric characterization of exact
scaled-polar optimality in the RR-ICM. We first recast Su's
fixed-sphere optimality condition in terms of the
Polar-Factor Admissibility Set, which records the positive polar
factors compatible with a scaled polar minimizer. We then characterize
this set exactly under an arbitrary radial law as a translated and
dilated spherical zonoid. The resulting formula identifies
atom--kink resonance as the sole source of nonscalar admissible
responses and leads to a quantitative spectral-capacity bound. As a
sharp consequence, every atomless radial law exhibits zero-capacity
rigidity: exact scaled-polar optimality can occur only when the
gradient's positive polar factor is scalar.

\subsection{From Exact Optimality to Polar-Factor Admissibility}
\label{subsec:RR-ICM}

Following the first-order condition used in Su’s proof~\cite{su2025isotropic}, we formulate the following exact admissibility criterion.
Let the compact
singular value decomposition and thin polar decomposition of the gradient be
\[
    G=U\Sigma V^\top
    =(UV^\top)(V\Sigma V^\top)
    =PS.
\]
The gradient \(G\) and curvature profile \(H\) induce the fixed-sphere
objective
\begin{equation*}
    J_{\mathrm{sph}}(Q;G)
    =
    -\langle G,Q\rangle
    +
    \mathbb{E}_{\zeta}H(\|Q\zeta\|_2),
    \qquad
    \zeta\sim\operatorname{Unif}(\mathbb{S}^{n-1}).
\end{equation*}
We ask when \(cP\) is a global minimizer of this objective. Su's result can be
written as
\[
    cP\in\operatorname*{arg\,min}_{Q}J_{\mathrm{sph}}(Q;G)
    \quad\Longleftrightarrow\quad
    S\in
    \left\{
        \mathbb{E}_{\zeta}
        [\eta(\zeta)\zeta\zeta^\top]:
        \eta:\mathbb S^{n-1}\to\partial H(c)
        \ \text{measurable}
    \right\}.
\]
Accordingly, define
\begin{equation}
    A(J_{\mathrm{sph}},cP)
    :=
    \left\{
        \mathbb{E}_{\zeta}
        [\eta(\zeta)\zeta\zeta^\top]:
        \eta:\mathbb S^{n-1}\to\partial H(c)
        \ \text{measurable}
    \right\},
    \label{eq:spherical-admissibility-definition}
\end{equation}
and call it the \emph{Polar-Factor Admissibility Set} determined by the
objective \(J_{\mathrm{sph}}\) and the evaluation point \(cP\). A wider subdifferential interval $\partial H(c)$ typically allows this
set to accommodate a richer family of positive polar factors. This
provides a geometric reading of Su's Assumption~3.3~\cite{su2025isotropic}: the slope interval
at the kink supplies the direction-dependent subgradient freedom needed
to match a broader range of gradient spectra, as is implicit in the
proof of the orthogonalization result.

Correspondingly, for the RR-ICM define
\begin{equation}
    A(J_R,cP)
    :=
    \left\{
        \mathbb E_{R,\zeta}
        [R\eta(R,\zeta)\zeta\zeta^\top]:
        \begin{array}{l}
        \eta\text{ measurable, and}\\[-0.2em]
        \eta(R,\zeta)\in\partial H(cR)
        \ \text{a.e. on }\{R>0\}
        \end{array}
    \right\}.
    \label{eq:rr-admissibility-definition}
\end{equation}
All integrands containing a factor \(R\) are defined to be zero on
\(\{R=0\}\). Our first objective is to characterize this set exactly.

Define the spherical zonoid
\begin{equation}
    \mathcal K_n
    :=
    \left\{
        \mathbb E_\zeta\!\left[a(\zeta)\zeta\zeta^\top\right]
        :
        a:\mathbb S^{n-1}\to[0,1]
        \text{ measurable}
    \right\}
    \subseteq \mathbb S^n,
    \label{eq:spherical-zonoid}
\end{equation}
where \(\mathbb S^n\) denotes the space of real symmetric \(n\times n\)
matrices. The selector \(a(\zeta)\) records the fraction of the slope jump
chosen in direction \(\zeta\). Thus, \(\mathcal K_n\) collects all
matrix-valued responses generated by direction-dependent measurable
selections from a scalar subdifferential interval. In the language of
classical zonoid theory, \(\mathcal K_n\) is the zonoid generated by the
random matrix \(\zeta\zeta^\top\in\mathbb S^n\); equivalently, it is the
Aumann expectation of the random line segment
\([0,\zeta\zeta^\top]\)~\cite{aumann1965integrals,koshevoy1998lift}.

For \(t>0\), let \(H_-(t)\) and \(H_+(t)\) denote the left and right
derivatives of \(H\), respectively, and define
\[
    \Delta H(t):=H_+(t)-H_-(t),
    \qquad
    \operatorname{Kink}(H)
    :=
    \{t>0:\Delta H(t)>0\}.
\]

\pagebreak[3]
\begin{theorem}
\label{thm:main1}
Assume that there exist \(p\geq1\) and \(C_H<\infty\)
such that $H_+(t)\leq C_H(1+t^{p-1})$,
and that \(\mathbb E[R^p]<\infty\). Define
\[
    b_c
    :=
    \frac1n\mathbb E[R H_-(cR)],
    \qquad
    w_c
    :=
    \sum_{\substack{r>0:\\cr\in\operatorname{Kink}(H)}}
    \mathbb P(R=r)\,r\Delta H(cr),
\]
where $c>0$. Then
\begin{equation}
    A(J_R,cP)
    =
    b_cI_n+w_c\mathcal K_n.
    \notag
\end{equation}
Consequently,
\begin{equation}
    cP\in\operatorname*{arg\,min}_{Q}J_R(Q;G)
    \quad\Longleftrightarrow\quad
    S\in b_cI_n+w_c\mathcal K_n.
    \notag
\end{equation}
\end{theorem}

\paragraph{Recovery of the fixed-sphere case.}
Theorem~\ref{thm:main1} contains Su's fixed-sphere result as an exact special
case. If \(R\equiv1\), then \(J_R=J_{\mathrm{sph}}\), and the only radius that
can contribute to the sum defining \(w_c\) is \(r=1\). Hence
\[
    b_c=\frac{H_-(c)}{n},
    \qquad
    w_c=\Delta H(c).
\]
Theorem~\ref{thm:main1} and the defining optimality interpretation of the
Polar-Factor Admissibility Set therefore give the complete specialization
\[
    cP\in\arg\min_Q J_{\mathrm{sph}}(Q;G)
    \quad\Longleftrightarrow\quad
    S\in A(J_{\mathrm{sph}},cP)
    =
    \frac{H_-(c)}{n}I_n+\Delta H(c)\mathcal K_n.
\]
Indeed, since \(\mathbb E_\zeta[\zeta\zeta^\top]=I_n/n\) and
\(\partial H(c)=[H_-(c),H_+(c)]\), every measurable selector in
\eqref{eq:spherical-admissibility-definition} has the form
\(\eta(\zeta)=H_-(c)+a(\zeta)\Delta H(c)\) with \(0\leq a(\zeta)\leq1\),
and conversely. Thus the displayed membership condition is precisely Su's
fixed-sphere exact-optimality criterion, rather than merely a specialization
of the two scalar coefficients.

\subsection{Zero-Capacity Rigidity under Atomless Radial Laws}
\label{subsec:spectral-capacity-barriers}

The coefficient \(w_c\) quantifies the total strength of atom--kink
resonance at scale \(c\). Each term in
\[
    w_c
    =
    \sum_{\substack{r>0:\\cr\in\operatorname{Kink}(H)}}
    \mathbb P(R=r)\,
    r\,\Delta H(cr)
\]
combines three ingredients: the atomic mass \(\mathbb P(R=r)\), the radial
chain-rule factor \(r\), and the slope jump
\(\Delta H(cr)=H_+(cr)-H_-(cr)\) available when the scaled radius \(cr\)
lands on a kink of \(H\). A radius contributes to \(w_c\) only if it is an
atom of the radial law and, after scaling by \(c\), aligns with a kink.
More importantly, \(w_c\) is the dilation coefficient of the spherical
zonoid in the exact admissibility region
\[
    b_cI_n+w_c\mathcal K_n.
\]
We therefore call \(w_c\) the \emph{atom--kink response capacity} at scale
\(c\). The term is convex-geometric: \(w_c\) controls the size of the
nonscalar matrix-response set generated by direction-dependent subgradient
selection. When \(w_c=0\), the response capacity vanishes and the exact
admissibility region collapses to the singleton \(\{b_cI_n\}\). For fixed
\(b_c\), increasing \(w_c\) dilates \(\mathcal K_n\) and enlarges the family
of positive polar factors that exact polar optimality can support.

To translate this convex-geometric capacity into a spectral quantity, note
that every \(K\in\mathcal K_n\) satisfies
\begin{equation}
    0
    \preceq
    K
    \preceq
    \frac1n I_n.
    \notag
\end{equation}
Indeed, if
\[
    K=\mathbb E_\zeta[a(\zeta)\zeta\zeta^\top],
    \qquad
    0\leq a(\zeta)\leq1,
\]
then
\[
    0
    \preceq
    a(\zeta)\zeta\zeta^\top
    \preceq
    \zeta\zeta^\top,
\]
and taking expectations gives
\[
    0
    \preceq
    K
    \preceq
    \mathbb E_\zeta[\zeta\zeta^\top]
    =
    \frac1nI_n.
\]
This motivates the dimension-normalized quantity
\begin{equation}
    \kappa_c
    :=
    \frac{w_c}{n},
    \notag
\end{equation}
which we call the \emph{spectral capacity} at scale \(c\). While \(w_c\)
measures the overall dilation of the matrix-response set, \(\kappa_c\)
controls the spectral variation that the set can support.

\begin{corollary}[Spectral-capacity bound for exact polar optimality]
\label{cor:exact-admissibility-spectral-capacity}
Under the assumptions of Theorem~\ref{thm:main1}, suppose that \(cP\) is an
exact minimizer. Then
\begin{equation}
    b_cI_n
    \preceq
    S
    \preceq
    (b_c+\kappa_c)I_n.
    \notag
\end{equation}
Define the spectral diameter of \(S\) by
\begin{equation}
    \operatorname{diam}_{\mathrm{eig}}(S)
    :=
    \lambda_{\max}(S)-\lambda_{\min}(S).
    \notag
\end{equation}
Then exact polar optimality necessarily implies
\begin{equation}
    \operatorname{diam}_{\mathrm{eig}}(S)
    \leq
    \kappa_c.
    \notag
\end{equation}
\end{corollary}
At a kink \(c\) of \(H\), the fixed-sphere law retains a nonzero spectral
response capacity because all of its radial mass resonates at the same kink.
By contrast, every atomless radial law satisfies \(w_c=0\); more generally,
only atoms aligned with scaled kinks contribute to \(w_c\).

\begin{corollary}[Zero-capacity rigidity under atomless radial laws]
\label{cor:atomless-zero-capacity-rigidity}
Under the assumptions of Theorem~\ref{thm:main1}, suppose that the law of
\(R\) is atomless. Then \(\kappa_c=0\). If \(cP\) is an exact minimizer,
then
\begin{equation}
    S=b_cI_n.
    \notag
\end{equation}
Consequently,
\(\operatorname{diam}_{\mathrm{eig}}(S)=0\).
\end{corollary}
In particular, if \(S\) has two distinct eigenvalues, then \(cP\) cannot be
an exact minimizer. Equivalently, whenever exact scaled-polar optimality
holds,
\begin{equation}
    G=b_cP.
    \notag
\end{equation}
Because \(G\) has full column rank, exact optimality forces \(b_c>0\). Thus, under an atomless radial law, exact scaled-polar optimality is
possible only in the degenerate spectral case where the gradient is
already a scaled column isometry.

\section{A Minimax Characterization of Muon's Polar Direction}
\label{sec:optimal}

\subsection{The Minimax RR-ICM} \label{subsec:minimax}
The rigidity result of Section~\ref{subsec:spectral-capacity-barriers}
shows that directional averaging in the RR-ICM may fail to support
exact polar optimality for a general gradient spectrum. Motivated by
this limitation, we retain the radial-curvature structure of the
RR-ICM while replacing the prescribed uniform directional law by
worst-case aggregation over all probability laws on the sphere.

Let $\mathcal P(\mathbb S^{n-1})$ denote the set of Borel probability
measures on $\mathbb S^{n-1}$. We define the
\emph{Minimax Random-Radius Isotropic Curvature Model}
(Minimax RR-ICM) by
\begin{equation}
    J_R^{\mathrm M}(Q;G)
    :=
    -\langle G,Q\rangle
    +
    \sup_{\nu\in\mathcal P(\mathbb S^{n-1})}
    \mathbb E_{\zeta\sim\nu}
    \overline H_R(\|Q\zeta\|_2).
    \label{eq:minimax-rr-icm}
\end{equation}
Thus, the RR-ICM and minimax RR-ICM share the same gradient term,
radial law, curvature profile, and directional state space; they differ
only in how directional variation is aggregated.

The ambiguity set in~\eqref{eq:minimax-rr-icm} is deliberately maximal.
Its purpose is not to estimate a data-driven uncertainty set, but to
identify the complete angular minimax counterpart of the RR-ICM~\cite{rahimian2022frameworks,kuhn2025dro}. Whenever a nonzero global minimizer exists, the resulting model
selects the polar factor as its normalized matrix direction for every
full-rank gradient spectrum.

\begin{theorem}[Minimax variational characterization of Muon's polar direction]
\label{thm:minimax-polar-reduction}
If the Minimax RR-ICM admits a nonzero global minimizer $Q^\star$, then
every such minimizer satisfies
\begin{equation}
    \frac{Q^\star}{\|Q^\star\|_{\mathrm{op}}}
    =
    P,
    \label{eq:universal-polar-direction}
\end{equation}
where $P$ is the polar factor of the gradient $G$.
\end{theorem}

Theorem~\ref{thm:minimax-polar-reduction} provides a new
curvature-based minimax variational characterization of Muon's polar
direction. In particular, the radial law and curvature profile may
affect whether a nonzero optimizer exists and what its scale is, but
they do not alter its normalized matrix direction. The next subsection
characterizes the corresponding scale condition and
Polar-Factor Admissibility Set.

\subsection{Minimax Admissibility and Trace Matching}
\label{subsec:minimax-admissibility}

Theorem~\ref{thm:minimax-polar-reduction} determines the normalized
matrix direction of every nonzero global minimizer, but does not yet
identify which scales along the polar ray are optimal. To characterize
the corresponding first-order condition, isolate the Minimax curvature
term
\begin{equation}
    \Phi_R^{\mathrm M}(Q)
    :=
    \sup_{\nu\in\mathcal P(\mathbb S^{n-1})}
    \mathbb E_{\zeta\sim\nu}
    \overline H_R(\|Q\zeta\|_2),
    \label{eq:minimax-curvature-term}
\end{equation}
so that
\[
    J_R^{\mathrm M}(Q;G)
    =
    -\langle G,Q\rangle+\Phi_R^{\mathrm M}(Q).
\]

For \(c>0\), define the Minimax Polar-Factor Admissibility Set by
\begin{equation}
    A(J_R^{\mathrm M},cP)
    :=
    \left\{
        M\in\mathbb S^n:
        PM\in\partial\Phi_R^{\mathrm M}(cP)
    \right\}.
    \label{eq:minimax-admissibility-definition}
\end{equation}
This is the Minimax counterpart of the admissibility set introduced for
the RR-ICM and records the positive polar factors compatible with exact
optimality at \(cP\).

Introduce the density-matrix set
\begin{equation}
    \mathcal D_n
    :=
    \left\{
        W\in\mathbb S_+^n:
        \operatorname{tr}(W)=1
    \right\},
    \label{eq:density-matrix-set}
\end{equation}
where \(\mathbb S_+^n\) denotes the cone of real symmetric positive
semidefinite \(n\times n\) matrices. The next theorem shows that the
matrix-valued admissibility condition reduces exactly to a scalar
trace-matching condition and yields the complete set of global
minimizers.

\begin{theorem}[Minimax admissibility and trace matching]
\label{thm:minimax-admissibility-trace}
For every \(c>0\), the Minimax Polar-Factor Admissibility Set is
\begin{align}
    A(J_R^{\mathrm M},cP)
    &=
    \left\{
        \alpha W:
        \alpha\in\partial\overline H_R(c),\
        W\in\mathcal D_n
    \right\}
    \notag\\
    &=
    \left\{
        M\in\mathbb S_+^n:
        \operatorname{tr}(M)
        \in\partial\overline H_R(c)
    \right\}.
    \label{eq:minimax-admissibility-trace-slab}
\end{align}
Moreover,
\begin{equation}
\begin{aligned}
    cP\in\operatorname*{arg\,min}_Q J_R^{\mathrm M}(Q;G)
    &\quad\Longleftrightarrow\quad
    S\in A(J_R^{\mathrm M},cP)
    \\
    &\quad\Longleftrightarrow\quad
    \|G\|_*
    \in\partial\overline H_R(c)
    \\
    &\quad\Longleftrightarrow\quad
    c\in\operatorname*{arg\,min}_{t\geq0}
    \left[
        \overline H_R(t)-t\|G\|_*
    \right].
\end{aligned}
\label{eq:minimax-trace-matching-equivalences}
\end{equation}
Consequently,
\begin{equation}
    \operatorname*{arg\,min}_{Q\in\mathbb R^{m\times n}}
    J_R^{\mathrm M}(Q;G)
    =
    \left\{
        cP:
        c\in\operatorname*{arg\,min}_{t\geq0}
        \left[
            \overline H_R(t)-t\|G\|_*
        \right]
    \right\}.
    \label{eq:minimax-complete-argmin}
\end{equation}
\end{theorem}
\begin{figure}[t]
    \centering
    \includegraphics[width=0.8\textwidth]
    {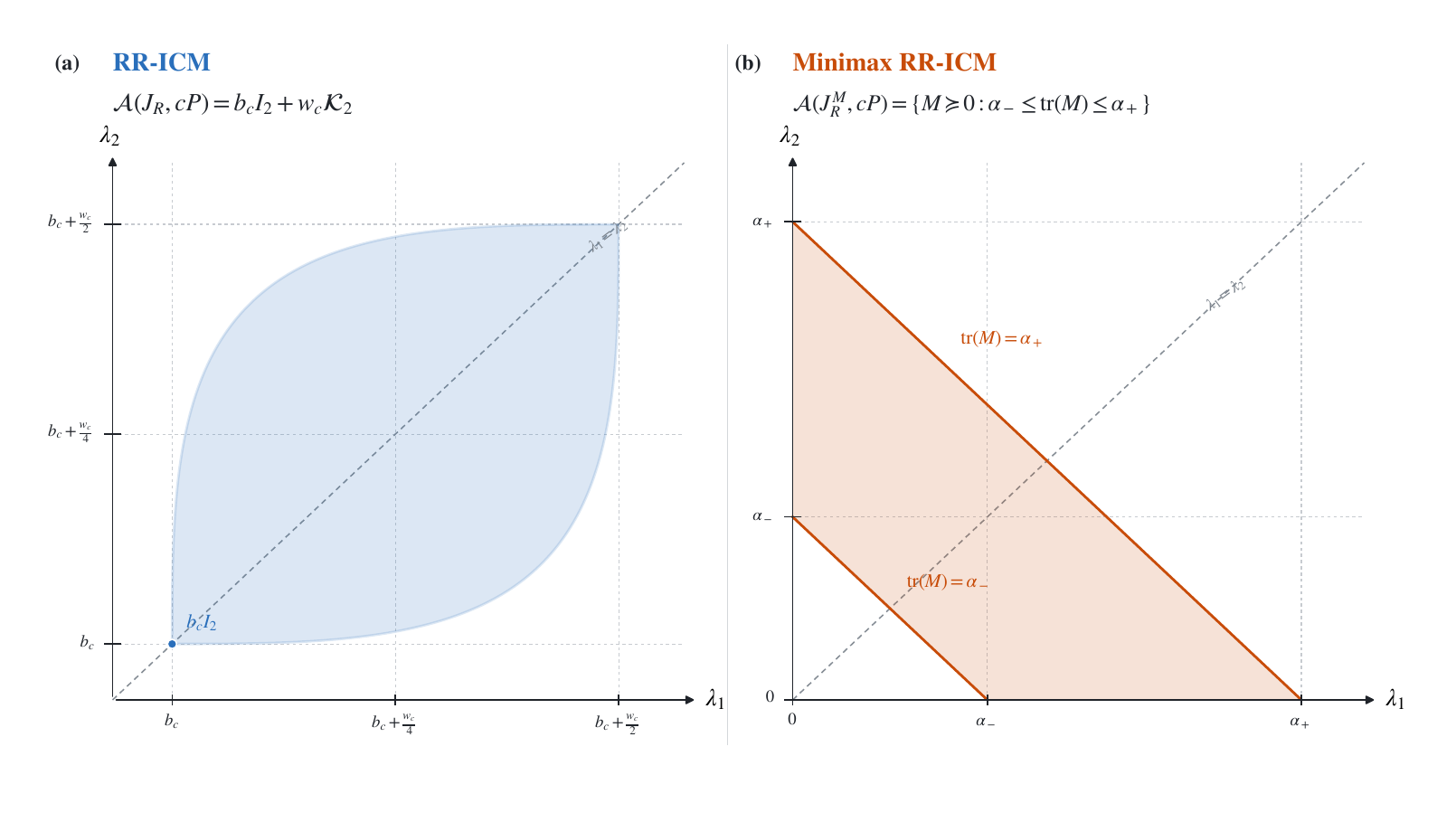}
    \caption{
    Two polar-factor admissibility geometries in dimension \(n=2\).
    Panel~\textnormal{(a)} shows the RR-ICM admissibility set
    \(b_cI_2+w_c\mathcal K_2\) in spectral coordinates. Its nonscalar
    extent is controlled by the atom--kink response capacity \(w_c\)
    and collapses to the singleton \(\{b_cI_2\}\) when \(w_c=0\).
    Panel~\textnormal{(b)} shows the Minimax RR-ICM admissibility set,
    where
    \(\alpha_-:=\min\partial\overline H_R(c)\) and
    \(\alpha_+:=\max\partial\overline H_R(c)\).
    Positive semidefiniteness confines the spectrum to the first
    quadrant, while trace matching restricts it to the slab
    \(\alpha_-\leq\lambda_1+\lambda_2\leq\alpha_+\).
    Thus, the RR-ICM restricts spectral shape through spherical-zonoid
    geometry, whereas the Minimax RR-ICM restricts only the total
    spectral mass.
    }
    \label{fig:two-admissibility-geometries}
\end{figure}

Theorem~\ref{thm:minimax-admissibility-trace} reveals an exact
shape--scale separation. Since
\[
    \frac{S}{\operatorname{tr}(S)}
    \in\mathcal D_n,
\]
the spectral shape of the positive polar factor \(S\) imposes no
additional restriction on Minimax admissibility. Exact optimality
reduces to the one-dimensional condition
\[
    \operatorname{tr}(S)
    =
    \|G\|_*
    \in\partial\overline H_R(c).
\]
Thus, the polar factor determines the optimal matrix direction, while
trace matching determines the admissible scale.

\subsection{Two Polar-Factor Admissibility Geometries}
\label{subsec:two-admissibility-geometries}

Combining Theorem~\ref{thm:main1} with
Theorem~\ref{thm:minimax-admissibility-trace} places the RR-ICM and the
Minimax RR-ICM within the same Polar-Factor Admissibility framework.
The following corollary isolates the distinct mechanisms through which
the two models support nonscalar gradient spectra.

\begin{corollary}[Two mechanisms of polar-factor admissibility]
\label{cor:two-admissibility-mechanisms}
For every \(c>0\),
\begin{align}
    A(J_R,cP)
    &=
    b_cI_n+w_c\mathcal K_n,
    \label{eq:comparison-average-admissibility}
    \\
    A(J_R^{\mathrm M},cP)
    &=
    \left\{
        M\in\mathbb S_+^n:
        \operatorname{tr}(M)
        \in\partial\overline H_R(c)
    \right\}.
    \label{eq:comparison-minimax-admissibility}
\end{align}
All nonscalar freedom in the RR-ICM is carried by
\(w_c\mathcal K_n\). In particular, if the law of \(R\) is atomless,
then
\[
    A(J_R,cP)=\{b_cI_n\}.
\]
By contrast, the Minimax RR-ICM admits every positive-semidefinite
matrix whose trace belongs to \(\partial\overline H_R(c)\). Whenever
\(n\geq2\) and \(\partial\overline H_R(c)\) contains a positive value,
this admissibility set contains nonscalar matrices without requiring
radial atoms to align with curvature kinks.
\end{corollary}
Corollary~\ref{cor:two-admissibility-mechanisms} provides the central
geometric bridge between the curvature-based ICM viewpoint and the
spectral geometry of Muon. In the RR-ICM, admissible nonscalar spectra
are confined to a translated and dilated spherical zonoid whose size
is controlled by atom--kink resonance. In the Minimax RR-ICM, the
admissible spectra instead fill a positive-semidefinite trace slab, so
only their total spectral mass is constrained.
The two models therefore generate nonscalar spectral freedom through
different mechanisms: atom--kink resonance in the RR-ICM, and the
set-valued subdifferential of the operator norm at the
equal-singular-value point \(cP\) in the Minimax RR-ICM. Consequently,
making the radial law atomless collapses the RR-ICM geometry but leaves
the matrix-shape freedom of the Minimax RR-ICM intact.

\section{Proofs}
\label{sec:proofs}

\subsection{Proofs for Section~\ref{subsec:RR-ICM}}

We use the standard finite-dimensional conventions for convex
subdifferentials, composition chain rules, and pointwise-supremum rules from
Rockafellar~\cite{rockafellar1970convex}.

\begin{proof}[Proof of Theorem~\ref{thm:main1}]
Retain the convention on the event \(\{R=0\}\). We first verify the
integrability needed to interchange expectation and subdifferentiation.
The growth assumption implies that, for every fixed \(C>0\),
\begin{equation}
    R H_+(CR)
    \leq
    C_H R+C_H C^{p-1}R^p.
    \label{eq:proof-all-scale-integrability}
\end{equation}
Since \(p\geq1\) and \(\mathbb E[R^p]<\infty\), we also have
\(\mathbb E[R]<\infty\). Thus, the right-hand side of
\eqref{eq:proof-all-scale-integrability} is integrable. In particular,
\(b_c\) and all matrix expectations below are finite.

Let
\[
    \Phi_R(Q)
    :=
    \mathbb E_{R,\zeta} H\!\left(R\lVert Q\zeta\rVert_2\right),
\]
so that \(J_R(Q;G)=-\langle G,Q\rangle+\Phi_R(Q)\). At \(Q=cP\),
the identity \(P^\top P=I_n\) gives \(\lVert cP\zeta\rVert_2=c\). For
\(R>0\), the subdifferential chain rule yields
\[
    \partial_Q H\!\left(R\lVert Q\zeta\rVert_2\right)\big|_{Q=cP}
    =
    \left\{
        R\eta\,P\zeta\zeta^\top:
        \eta\in\partial H(cR)
    \right\}.
\]
If \(\lVert Q-cP\rVert_{\mathrm{op}}<\varepsilon<c\), then
\[
    (c-\varepsilon)R
    \leq R\lVert Q\zeta\rVert_2
    \leq (c+\varepsilon)R.
\]
Local subgradients on this neighborhood admit a uniform integrable bound.
Indeed, for \(R>0\), take any
\[
    V\in\partial_Q H\!\left(R\lVert Q\zeta\rVert_2\right),
\]
Since \(c-\varepsilon>0\), we have \(Q\zeta\neq0\), and there exists
\(
\eta\in\partial H(R\lVert Q\zeta\rVert_2)
\)
such that

\[
    V
    =
    R\eta\,
    \frac{Q\zeta}{\lVert Q\zeta\rVert_2}\zeta^\top.
\]
Because \(H\) is convex and nondecreasing, its subgradients are nonnegative
and its right derivative is nondecreasing. Therefore,
\begin{align}
    \lVert V\rVert_F
    &=R\eta \notag\\
    &\leq R H_+\!\left((c+\varepsilon)R\right) \notag\\
    &\leq C_HR+C_H(c+\varepsilon)^{p-1}R^p.
    \label{eq:local-subgradient-integrable-bound}
\end{align}
When $R=0$, the integrand is the constant $H(0)$, and its
$Q$-subgradient is zero. The right-hand side of \eqref{eq:local-subgradient-integrable-bound}
is integrable.
Moreover, the mapping
\[
(R,\zeta,Q)
\longmapsto
H\bigl(R\|Q\zeta\|_2\bigr)
\]
is measurable in $(R,\zeta)$ for every fixed $Q$, and is continuous
and convex in $Q$ for every fixed $(R,\zeta)$. 
Moreover, for
$\|Q-cP\|_{\mathrm{op}}<\varepsilon$,
\[
    \left|
        H\!\left(R\|Q\zeta\|_2\right)
    \right|
    \leq
    |H(0)|
    +
    (c+\varepsilon)R
    H_+\!\left((c+\varepsilon)R\right),
\]
whose expectation is finite by the same growth and moment assumptions.
Hence it is a convex
normal integrand. 

Together with the integrable local Lipschitz modulus
established above, the hypotheses of the interchange theorem for
subdifferentials of integral convex functionals
\cite{rockafellarwets1982interchange}
are satisfied. Therefore,
\[
\partial\Phi_R(cP)
=
\left\{
\mathbb E_{R,\zeta}[V(R,\zeta)]:
V(R,\zeta)\in
\partial_Q H\bigl(R\|Q\zeta\|_2\bigr)\big|_{Q=cP}
\ \text{a.s.},
\quad
V\ \text{measurable}
\right\}.
\]
Using the pointwise subdifferential formula above and the definition
of the Polar-Factor Admissibility Set, we obtain
\begin{equation}
    A(J_R,cP)
    =
    \left\{
    \mathbb E_{R,\zeta}
    \left[R\eta(R,\zeta)\zeta\zeta^\top\right]:
    \begin{array}{l}
    \eta\text{ measurable},\\[-0.2em]
    \eta(R,\zeta)\in\partial H(cR)
    \ \text{a.e. on }\{R>0\}
    \end{array}
    \right\}.
    \label{eq:admissibility-selector-representation}
\end{equation}
Equivalently, \(\partial\Phi_R(cP)=P A(J_R,cP)\). Thus, when \(G=PS\),
the first-order optimality condition can be written as
\[
    cP\in\operatorname*{arg\,min}_Q J_R(Q;G)
    \quad\Longleftrightarrow\quad
    S\in A(J_R,cP).
\]

Because \(H\) is a finite convex function,
\[
    \partial H(t)
    =
    [H_-(t),H_+(t)]
    =
    H_-(t)+\Delta H(t)[0,1],
    \qquad t>0.
\]
Every selector in~\eqref{eq:admissibility-selector-representation} can be
written as
\[
    \eta(R,\zeta)
    =
    H_-(cR)+a(R,\zeta)\Delta H(cR),
    \qquad 0\leq a(R,\zeta)\leq 1.
\]
By independence of \(R\) and \(\zeta\) and spherical symmetry,
\begin{align}
    \mathbb E_{R,\zeta}
    \left[R H_-(cR)\zeta\zeta^\top\right]
    &=
    \mathbb E[R H_-(cR)]\,
    \mathbb E_\zeta[\zeta\zeta^\top] \notag\\
    &=b_c I_n.
    \label{eq:proof-admissibility-baseline}
\end{align}
where we used \(\mathbb E_\zeta[\zeta\zeta^\top]=I_n/n\).

Now let \(W_c:=R\Delta H(cR)\). A finite convex function has at most
countably many nondifferentiability points, and \(W_c\) can be nonzero only
when \(R>0\) and \(cR\in\operatorname{Kink}(H)\). Therefore,
\begin{align}
    \mathbb E[W_c]
    &=
    \sum_{\substack{r>0:\\cr\in\operatorname{Kink}(H)}}
    \mathbb P(R=r)\,r\Delta H(cr)
    =w_c.
    \label{eq:proof-resonance-expectation}
\end{align}
Moreover, \(0\leq W_c\leq R H_+(cR)\), so \(W_c\) is integrable.

If \(w_c=0\), then \(W_c=0\) almost surely and the nonbaseline part of
\eqref{eq:admissibility-selector-representation} vanishes. If \(w_c>0\),
then for any measurable selector \(a(R,\zeta)\in[0,1]\), define
\[
    \bar a(\zeta)
    :=
    \frac{\mathbb E_R[W_c a(R,\zeta)]}{w_c}.
\]
Then \(\bar a:\mathbb S^{n-1}\to[0,1]\) is measurable, and Fubini's theorem~\cite{royden1988real}
gives
\begin{align}
    \mathbb E_{R,\zeta}
    \left[W_c a(R,\zeta)\zeta\zeta^\top\right]
    &=
    w_c\,
    \mathbb E_\zeta
    \left[\bar a(\zeta)\zeta\zeta^\top\right]
    \in w_c\mathcal K_n.
    \label{eq:proof-admissibility-remainder}
\end{align}
Conversely, for any \(\bar a:\mathbb S^{n-1}\to[0,1]\), choosing
\(a(R,\zeta)=\bar a(\zeta)\) realizes the corresponding arbitrary element
of \(w_c\mathcal K_n\). Combining this observation with
\eqref{eq:proof-admissibility-baseline} yields
\[
    A(J_R,cP)=b_cI_n+w_c\mathcal K_n,
\]
as claimed.
\end{proof}

\subsection{Proofs for Section~\ref{subsec:spectral-capacity-barriers}}

\begin{proof}[Proof of Corollary~\ref{cor:exact-admissibility-spectral-capacity}]
By exact optimality of \(cP\), the first-order condition, and
Theorem~\ref{thm:main1}, there exists \(K\in\mathcal K_n\) such that
\begin{equation}
    S=b_cI_n+w_cK.
    \label{eq:proof-spectral-decomposition}
\end{equation}
By the definition of the spherical zonoid, there is a measurable function
\(a:\mathbb S^{n-1}\to[0,1]\) such that
\[
    K=\mathbb E_\zeta[a(\zeta)\zeta\zeta^\top].
\]
Hence,
\[
    0\preceq K
    \preceq\mathbb E_\zeta[\zeta\zeta^\top]
    =\frac1nI_n.
\]
Since \(w_c\geq0\), substituting this Loewner-order bound into
\eqref{eq:proof-spectral-decomposition} and using \(\kappa_c=w_c/n\) gives
\[
    b_cI_n\preceq S\preceq(b_c+\kappa_c)I_n.
\]
Thus \(\lambda_{\min}(S)\geq b_c\) and
\(\lambda_{\max}(S)\leq b_c+\kappa_c\), so
\[
    \operatorname{diam}_{\mathrm{eig}}(S)
    =\lambda_{\max}(S)-\lambda_{\min}(S)
    \leq\kappa_c.
\]
Both conclusions follow.
\end{proof}

\begin{proof}[Proof of Corollary~\ref{cor:atomless-zero-capacity-rigidity}]
The kink set \(\operatorname{Kink}(H)\) of a finite convex function is at
most countable. Hence the resonant-radius set
\[
    \mathcal R_c
    :=
    \{r>0:cr\in\operatorname{Kink}(H)\}
\]
is also at most countable. If the law of \(R\) is atomless, then
\(\mathbb P(R=r)=0\) for every \(r\in\mathcal R_c\). By the definition of
\(w_c\),
\[
    w_c
    =
    \sum_{r\in\mathcal R_c}
    \mathbb P(R=r)\,r\Delta H(cr)
    =0,
\]
Thus \(\kappa_c=w_c/n=0\). If \(cP\) is an exact minimizer,
Theorem~\ref{thm:main1} further gives
\[
    S\in b_cI_n+w_c\mathcal K_n=\{b_cI_n\}.
\]
Therefore \(S=b_cI_n\), all of its eigenvalues coincide, and
\[
    \operatorname{diam}_{\mathrm{eig}}(S)=0.
\]
\end{proof}

\subsection{Proofs for Section~\ref{subsec:minimax}}
\label{sec:proofs-minimax-direction}

The proof of Theorem~\ref{thm:minimax-polar-reduction} has two independent
steps. The first eliminates the supremum over directional laws. The second
shows that, at fixed operator norm, the gradient's linear term is uniquely
maximized by the polar direction. We isolate these steps in two lemmas.

\begin{lemma}[Operator-norm representation of the worst direction]
\label{lem:minimax-worst-direction}
For every \(Q\in\mathbb R^{m\times n}\),
\begin{equation}
    \sup_{\nu\in\mathcal P(\mathbb S^{n-1})}
    \mathbb E_{\zeta\sim\nu}
    \overline H_R(\lVert Q\zeta\rVert_2)
    =
    \overline H_R(\lVert Q\rVert_{\mathrm{op}}).
    \label{eq:minimax-sup-opnorm}
\end{equation}
\end{lemma}

\begin{proof}
Let \(t:=\lVert Q\rVert_{\mathrm{op}}\). For every
\(\zeta\in\mathbb S^{n-1}\), the definition of the operator norm gives
\(\lVert Q\zeta\rVert_2\leq t\). By monotonicity of \(\overline H_R\), for
every \(\nu\in\mathcal P(\mathbb S^{n-1})\),
\[
    \mathbb E_{\zeta\sim\nu}
    \overline H_R(\lVert Q\zeta\rVert_2)
    \leq \overline H_R(t).
\]
Thus, \(\overline H_R(t)\) is an upper bound for the supremum. To show that
it is attained, compactness of the finite-dimensional unit sphere provides
\(\zeta_Q\in\mathbb S^{n-1}\) such that
\(\lVert Q\zeta_Q\rVert_2=t\). Since the Dirac measure
\(\delta_{\zeta_Q}\) belongs to \(\mathcal P(\mathbb S^{n-1})\),
\[
    \mathbb E_{\zeta\sim\delta_{\zeta_Q}}
    \overline H_R(\lVert Q\zeta\rVert_2)
    =\overline H_R(t).
\]
The upper bound is therefore attained, proving
\eqref{eq:minimax-sup-opnorm}.
\end{proof}

\begin{lemma}[Polar-factor domination]
\label{lem:minimax-polar-domination}
For any \(Q\in\mathbb R^{m\times n}\), let
\(t:=\lVert Q\rVert_{\mathrm{op}}\). Then
\begin{equation}
    \langle G,Q\rangle
    \leq t\lVert G\rVert_*
    =\langle G,tP\rangle.
    \label{eq:polar-linear-domination}
\end{equation}
Equality in~\eqref{eq:polar-linear-domination} holds if and only if
\(Q=tP\). Consequently,
\begin{equation}
    J_R^{\mathrm M}(Q;G)
    \geq J_R^{\mathrm M}(tP;G),
    \label{eq:minimax-polar-domination}
\end{equation}
and equality in~\eqref{eq:minimax-polar-domination} holds if and only if
\(Q=tP\).
\end{lemma}

\begin{proof}
Take the compact singular value decomposition
\[
    G=U\Sigma V^\top,
    \qquad
    \Sigma=\operatorname{diag}(\sigma_1,\ldots,\sigma_n).
\]
The standing full-column-rank assumption gives \(\sigma_i>0\), and the polar
factor is \(P=UV^\top\). Let \(u_i\) and \(v_i\) be the corresponding left
and right singular vectors. We seek to maximize \(\langle G,Q\rangle\) over
matrices with operator norm at most \(t\). The singular value decomposition
gives
\begin{align}
    \langle G,Q\rangle
    &=\sum_{i=1}^n \sigma_i u_i^\top Qv_i \notag\\
    &\leq\sum_{i=1}^n \sigma_i\lVert Qv_i\rVert_2 \notag\\
    &\leq t\sum_{i=1}^n\sigma_i
    =t\lVert G\rVert_*.
    \label{eq:proof-polar-linear-bound}
\end{align}
On the other hand, \(Pv_i=u_i\), so
\[
    \langle G,tP\rangle
    =t\sum_{i=1}^n\sigma_i
    =t\lVert G\rVert_*.
\]
This proves~\eqref{eq:polar-linear-domination}.

It remains to identify the unique equality case; this is where full column
rank is essential. If \(t=0\), then \(Q=0=tP\). Suppose \(t>0\). Since every
\(\sigma_i\) is strictly positive, equality in
\eqref{eq:proof-polar-linear-bound} holds if and only if, for every
\(i=1,\ldots,n\),
\[
    u_i^\top Qv_i=\lVert Qv_i\rVert_2=t.
\]
The equality condition in the Cauchy--Schwarz inequality then gives
\(Qv_i=tu_i\). Since \(v_1,\ldots,v_n\) form an orthonormal basis of
\(\mathbb R^n\),
\[
    Q=t\sum_{i=1}^n u_iv_i^\top=tUV^\top=tP.
\]
Conversely, \(Q=tP\) makes every inequality above an equality. Thus, equality
in the linear domination bound holds exactly when \(Q=tP\).

Finally, Lemma~\ref{lem:minimax-worst-direction} and
\(\lVert tP\rVert_{\mathrm{op}}=t\) show that \(Q\) and \(tP\) have the same
worst-direction curvature value. Substituting
\eqref{eq:polar-linear-domination} into the objective proves
\eqref{eq:minimax-polar-domination}, with the same equality condition.
\end{proof}

\pagebreak[3]
\begin{proof}[Proof of Theorem~\ref{thm:minimax-polar-reduction}]
Let \(Q^\star\neq0\) be a global minimizer and set
\(t^\star:=\lVert Q^\star\rVert_{\mathrm{op}}>0\). Polar-factor domination
and global optimality give, respectively,
\[
    J_R^{\mathrm M}(t^\star P;G)
    \leq J_R^{\mathrm M}(Q^\star;G)
    \leq J_R^{\mathrm M}(t^\star P;G).
\]
The two endpoints must therefore be equal. By the unique equality case in
Lemma~\ref{lem:minimax-polar-domination}, \(Q^\star=t^\star P\). Dividing by
\(t^\star=\lVert Q^\star\rVert_{\mathrm{op}}\) gives
\[
    \frac{Q^\star}{\lVert Q^\star\rVert_{\mathrm{op}}}=P,
\]
as claimed.
\end{proof}

\subsection{Proofs for Section~\ref{subsec:minimax-admissibility}}

We first isolate the central geometric fact: at a scaled polar point, all
subgradients of the operator norm are parameterized by density matrices. This
formula is the equal-singular-value specialization of classical
subdifferential characterizations for unitarily invariant matrix norms
~\cite{watson1992characterization,lewis1995convex}. We include a direct proof
for completeness.

\begin{lemma}[Operator-norm subdifferential at an equal-singular-value point]
\label{lem:opnorm-subdiff-polar}
For every \(c>0\),
\begin{equation}
    \partial\lVert\cdot\rVert_{\mathrm{op}}(cP)
    =
    P\mathcal D_n.
    \label{eq:opnorm-subdiff-polar}
\end{equation}
\end{lemma}

\begin{proof}
The operator norm is the pointwise maximum of linear functions indexed by a
compact set:
\begin{equation}
    \lVert Q\rVert_{\mathrm{op}}
    =
    \max_{\substack{x\in\mathbb S^{n-1}\\
                    y\in\mathbb S^{m-1}}}
    \langle Q,yx^\top\rangle.
    \label{eq:opnorm-support-representation}
\end{equation}
At \(Q=cP\), the identity \(P^\top P=I_n\) gives
\[
    \langle cP,yx^\top\rangle
    =
    c\,y^\top Px
    \leq c.
\]
Because \(Px\) is a unit vector, equality holds if and only if \(y=Px\).
Thus, the gradients of all active linear functions in
\eqref{eq:opnorm-support-representation} at \(cP\) are
\[
    (Px)x^\top=Pxx^\top,
    \qquad x\in\mathbb S^{n-1}.
\]
The subdifferential of a pointwise maximum is the convex hull of the active
gradients. Hence,
\[
    \partial\lVert\cdot\rVert_{\mathrm{op}}(cP)
    =
    P\operatorname{conv}
    \left\{
        xx^\top:x\in\mathbb S^{n-1}
    \right\}.
\]
On the other hand, the spectral decomposition of positive semidefinite
matrices gives
\[
    \operatorname{conv}
    \left\{
        xx^\top:x\in\mathbb S^{n-1}
    \right\}
    =
    \left\{
        W\succeq0:\operatorname{tr}(W)=1
    \right\}
    =
    \mathcal D_n.
\]
Substitution proves~\eqref{eq:opnorm-subdiff-polar}.
\end{proof}

\begin{proof}[Proof of Theorem~\ref{thm:minimax-admissibility-trace}]
The proof has three steps: computing the subdifferential of the Minimax
curvature term at \(cP\), reducing matrix optimality to trace matching, and
recovering the complete minimizer set.

\emph{Step 1: compute the admissibility set.}
By Lemma~\ref{lem:minimax-worst-direction},
\begin{equation}
    \Phi_R^{\mathrm M}(Q)
    =
    \overline H_R(\lVert Q\rVert_{\mathrm{op}}).
    \label{eq:proof-minimax-curvature-opnorm}
\end{equation}
The function \(\overline H_R\) is finite, convex, and nondecreasing. For
\(c>0\), the convex-composition chain rule and
Lemma~\ref{lem:opnorm-subdiff-polar} give
\begin{align}
    \partial\Phi_R^{\mathrm M}(cP)
    &=
    \left\{
        \alpha Z:
        \alpha\in\partial\overline H_R(c),\
        Z\in\partial\lVert\cdot\rVert_{\mathrm{op}}(cP)
    \right\} \notag\\
    &=
    P\left\{
        \alpha W:
        \alpha\in\partial\overline H_R(c),\
        W\in\mathcal D_n
    \right\}.
    \label{eq:proof-minimax-curvature-subdiff}
\end{align}
Combining this identity with~\eqref{eq:minimax-admissibility-definition} and
\(P^\top P=I_n\) yields
\begin{equation}
    A(J_R^{\mathrm M},cP)
    =
    \left\{
        \alpha W:
        \alpha\in\partial\overline H_R(c),\
        W\in\mathcal D_n
    \right\}.
    \label{eq:proof-minimax-admissibility-density}
\end{equation}

It remains to rewrite this density-matrix representation as an intrinsic
positive-semidefinite trace slab. Since \(\overline H_R\) is nondecreasing,
\(\partial\overline H_R(c)\subseteq[0,\infty)\). If \(M=\alpha W\) belongs
to the right-hand side of
\eqref{eq:proof-minimax-admissibility-density}, then
\[
    M\succeq0,
    \qquad
    \operatorname{tr}(M)=\alpha
    \in\partial\overline H_R(c).
\]
Conversely, suppose \(M\succeq0\) and
\(\alpha:=\operatorname{tr}(M)\in\partial\overline H_R(c)\). If
\(\alpha>0\), take \(W=M/\alpha\in\mathcal D_n\). If \(\alpha=0\), positive
semidefiniteness forces \(M=0\), and any \(W\in\mathcal D_n\) still satisfies
\(M=\alpha W\). This proves
\eqref{eq:minimax-admissibility-trace-slab}.

\emph{Step 2: reduce matrix optimality to trace matching.}
Since \(J_R^{\mathrm M}\) is convex, its first-order optimality condition at
\(cP\) is
\[
    cP\in\operatorname*{arg\,min}_Q J_R^{\mathrm M}(Q;G)
    \quad\Longleftrightarrow\quad
    G\in\partial\Phi_R^{\mathrm M}(cP).
\]
Writing \(G=PS\) and using
\eqref{eq:proof-minimax-curvature-subdiff} together with
\(P^\top P=I_n\), we obtain
\[
    G\in\partial\Phi_R^{\mathrm M}(cP)
    \quad\Longleftrightarrow\quad
    S\in A(J_R^{\mathrm M},cP).
\]
The standing assumptions ensure \(S\succ0\). Therefore,
\eqref{eq:minimax-admissibility-trace-slab} gives
\[
    S\in A(J_R^{\mathrm M},cP)
    \quad\Longleftrightarrow\quad
    \operatorname{tr}(S)
    \in\partial\overline H_R(c).
\]
Since \(S=(G^\top G)^{1/2}\), we have
\(\operatorname{tr}(S)=\lVert G\rVert_*\), proving equivalence of the first
three conditions.

Now define the scalar convex function
\[
    \psi_G(t)
    :=
    \overline H_R(t)-t\lVert G\rVert_*,
    \qquad t\geq0.
\]
Because \(c>0\) is an interior point of the domain, scalar convex optimality
gives
\[
    c\in\operatorname*{arg\,min}_{t\geq0}\psi_G(t)
    \quad\Longleftrightarrow\quad
    0\in\partial\overline H_R(c)-\lVert G\rVert_*
    \quad\Longleftrightarrow\quad
    \lVert G\rVert_*\in\partial\overline H_R(c).
\]
This proves~\eqref{eq:minimax-trace-matching-equivalences}.

\emph{Step 3: recover the complete minimizer set.}
For any \(Q\), let \(t=\lVert Q\rVert_{\mathrm{op}}\).
Lemma~\ref{lem:minimax-polar-domination} and
\(\lVert tP\rVert_{\mathrm{op}}=t\) give
\begin{equation}
    J_R^{\mathrm M}(Q;G)
    \geq
    J_R^{\mathrm M}(tP;G)
    =
    \overline H_R(t)-t\lVert G\rVert_*
    =
    \psi_G(t),
    \label{eq:proof-minimax-scalar-reduction}
\end{equation}
with equality if and only if \(Q=tP\). If \(c\) minimizes \(\psi_G\), then
\eqref{eq:proof-minimax-scalar-reduction} shows that \(cP\) minimizes
\(J_R^{\mathrm M}\). Conversely, suppose \(Q^\star\) is a global minimizer
of \(J_R^{\mathrm M}\), and let
\(t^\star=\lVert Q^\star\rVert_{\mathrm{op}}\). Polar-factor domination and
global optimality force
\[
    J_R^{\mathrm M}(t^\star P;G)
    =
    J_R^{\mathrm M}(Q^\star;G).
\]
By the unique equality condition in
Lemma~\ref{lem:minimax-polar-domination}, \(Q^\star=t^\star P\). Comparing
with arbitrary \(sP\) further shows that \(t^\star\) must minimize
\(\psi_G\). Hence~\eqref{eq:minimax-complete-argmin} holds.
\end{proof}

\subsection{Proofs for Section~\ref{subsec:two-admissibility-geometries}}

\begin{proof}[Proof of Corollary~\ref{cor:two-admissibility-mechanisms}]
Equations~\eqref{eq:comparison-average-admissibility}
and~\eqref{eq:comparison-minimax-admissibility} follow directly from
Theorems~\ref{thm:main1} and~\ref{thm:minimax-admissibility-trace}. If the law
of \(R\) is atomless, the proof of
Corollary~\ref{cor:atomless-zero-capacity-rigidity} shows that \(w_c=0\).
Therefore,
\[
    A(J_R,cP)=b_cI_n+0\mathcal K_n=\{b_cI_n\}.
\]
This proves that all nonscalar freedom in the average model is carried by
\(w_c\mathcal K_n\).

Now suppose \(n\geq 2\), and choose
\(\alpha>0\) such that
\(\alpha\in\partial\overline H_R(c)\). Let
\[
W_1
:=
\frac{1}{n+1}
\operatorname{diag}(2,1,\ldots,1),
\qquad
W_2
:=
\frac{1}{n}I_n.
\]
Both \(W_1\) and \(W_2\) belong to \(\mathcal D_n\).
Therefore, by Theorem~3,
\[
M_1:=\alpha W_1,
\qquad
M_2:=\alpha W_2
\]
both belong to
\(\mathcal A(J_R^{\mathrm M},cP)\).
Moreover, \(M_1\) is positive definite and nonscalar,
whereas \(M_2\) is scalar.
This construction relies only on positive semidefiniteness
and the trace constraint; it requires neither radial atoms
nor positive probability that \(cR\) hits a kink of \(H\).
Thus, the matrix-shape freedom in the Minimax model does not
depend on atom--kink resonance.
\end{proof}

\section{Discussion and Conclusion}
\label{sec:discussion}

This paper studied the variational status of Muon's polar direction
within curvature-based one-step models. Starting from the Isotropic
Curvature Model, we asked whether its fixed-sphere explanation of
gradient orthogonalization remains valid when the input radius is
allowed to vary. The resulting analysis reveals a sharp distinction
between radial concentration at the level of function values and
stability of the first-order geometry governing exact optimality.
To make this distinction precise, we introduced the Random-Radius
Isotropic Curvature Model and the Polar-Factor Admissibility Set. The
latter converts exact scaled-polar optimality into a geometric
condition on the positive polar factor of the gradient. For an
arbitrary radial law, we obtained a complete characterization of this
set as a translated and dilated spherical zonoid. The translation
represents the isotropic baseline response, whereas the dilation
contains all available nonscalar matrix freedom.
The key quantity controlling this dilation is the atom--kink response
capacity. It receives contributions only from radial atoms that align,
after scaling, with nondifferentiable points of the curvature profile.
Thus, nonscalar admissibility in the RR-ICM is not generated by radial
concentration alone, but by an exact resonance between the radial law
and the curvature kinks. This characterization yields a quantitative
spectral-capacity bound and, in the atomless case, a sharp rigidity
result: exact scaled-polar optimality can hold only when the positive
polar factor of the gradient is scalar, or equivalently when the
gradient is already a scaled column isometry.
This rigidity highlights a singular feature of exact first-order
optimality. A thin atomless shell may be arbitrarily close to a
deterministic sphere, and the corresponding effective curvature
profiles may be close in function value, while their admissibility
geometries remain fundamentally different. The deterministic sphere
can place all directional samples at the same kink and exploit the
resulting interval of subgradients. An atomless radial perturbation
hits each prescribed kink with probability zero and can therefore
remove the entire direction-dependent slope-selection mechanism.
Radial concentration does not, by itself, preserve the variational
structure required for exact orthogonalization.
This conclusion should not be interpreted as a limitation of Muon as
an optimization method. Rather, it identifies a limitation of
directional averaging as a curvature-based explanation of Muon's polar
update. The empirical effectiveness and established spectral
interpretations of Muon suggest that its polar direction reflects a
geometry more robust than the atom--kink mechanism available in the
RR-ICM.

Motivated by this tension, we introduced the Minimax RR-ICM, which
retains the radial law, curvature profile, and directional state space
of the RR-ICM, but replaces the prescribed uniform directional average
by worst-case aggregation over probability laws on the sphere. We
proved that, whenever a nonzero global minimizer exists, every such
minimizer lies on the polar ray generated by the gradient. The radial
law and curvature profile determine whether an optimizer exists and
select its scalar magnitude, but they do not alter its normalized
matrix direction. The scale is characterized by an exact
one-dimensional convex problem involving the effective radial profile
and the nuclear norm of the gradient.
The Minimax RR-ICM therefore provides a new curvature-based minimax
variational characterization of Muon's polar direction. Importantly,
the operator-norm geometry is not imposed through an external
constraint, trust region, or norm budget. It emerges from robustifying
the same radial-curvature model over directional probability laws.
This supplies a direct variational bridge between the ICM viewpoint
and existing spectral interpretations of Muon.
The two models also admit a unified comparison through the
Polar-Factor Admissibility Set. In the RR-ICM, admissible nonscalar
spectra are confined to a translated and dilated spherical zonoid,
whose size is governed by atom--kink resonance. In the Minimax RR-ICM,
the admissibility set is instead a positive-semidefinite trace slab.
The latter constrains the total spectral mass but accommodates every
positive-semidefinite spectral shape with an admissible trace.
Accordingly, the two models generate nonscalar freedom through
different mechanisms: direction-dependent slope selection at radial
atom--kink resonances in the RR-ICM, and the set-valued
subdifferential geometry of the operator norm at a scaled polar point
in the Minimax RR-ICM. Radial atomlessness destroys the former
mechanism but leaves the latter matrix-shape freedom intact.

\paragraph{Scope and future directions.}
The present framework concerns exact optimality in a one-step
curvature surrogate. It therefore does not by itself determine whether
the radial fragility identified here leads to a substantial
optimization error, nor does it describe the full training dynamics of
practical Muon.
The first natural direction is to develop a quantitative theory of
approximate polar-factor admissibility. Although an atomless radial
law eliminates the set-valued response required for exact
scaled-polar optimality, a highly concentrated radial law may still
make the polar update nearly optimal or place the true minimizer close
to the polar ray. Relating the distance
\[
    \operatorname{dist}
    \bigl(S, A(J_R,cP)\bigr)
\]
to an objective gap, a directional error, or the distance between the
optimizer and the polar ray would determine whether radial fragility
is only an exact first-order phenomenon or persists quantitatively.
A second direction is to study directional ambiguity sets between the
two endpoints considered in this paper. The RR-ICM fixes the uniform
spherical law, whereas the Minimax RR-ICM optimizes over all
probability laws on the sphere. Restricted ambiguity classes, such as
divergence or transport neighborhoods of the uniform law, may generate
intermediate Polar-Factor Admissibility Sets and interpolate between
spherical-zonoid and positive-semidefinite trace-slab geometry.
Characterizing this interpolation could quantify how much directional
robustification is required to support a prescribed range of gradient
spectra.

More broadly, our results show that the variational geometry selected
by an isotropic curvature model depends critically on how variation
across perturbation directions is aggregated. Directional averaging
makes exact polar optimality radially fragile, whereas full minimax
aggregation makes the polar direction universal, conditional on the
existence of a nonzero optimizer. The Polar-Factor Admissibility Set
provides a common language for studying both endpoints and the
geometries that may lie between them.

\bibliographystyle{abbrv}
\bibliography{ref.bib}

\end{document}